\documentclass{article}

\usepackage[preprint]{neurips_2023}

\usepackage{enumitem}
\usepackage[utf8]{inputenc} 
\usepackage[T1]{fontenc} 
\usepackage{natbib}
\usepackage{hyperref}       
\usepackage{url}            
\usepackage{booktabs}       
\usepackage{amsfonts}       
\usepackage{nicefrac}       
\usepackage{microtype}      
\usepackage{xcolor}         
\usepackage{amsmath,amsthm, amssymb, mathtools ,bbm}
\usepackage{graphicx}
\usepackage{float, subfig, wrapfig}

\usepackage{algorithm}
\usepackage{algorithmicx}
\usepackage{algpseudocode}

\usepackage{tcolorbox}

\newtheorem{theorem}{Theorem}
\newtheorem{corollary}[theorem]{Corollary}
\newtheorem{lemma}[theorem]{Lemma}

\newtheorem{example}[theorem]{Example}
\newtheorem{definition}[theorem]{Definition}

\newtheorem{assumption}[theorem]{Assumption}

\DeclarePairedDelimiterX{\innerprod}[2]{\langle}{\rangle}{#1, #2}

\newcommand{\norm}[2]{\left\| #1 \right\|_{#2}}
\newcommand{\grad}{\nabla}

\newcommand{\R}{\mathbb{R}}

\newcommand{\E}{\mathbb{E}}
\newcommand{\mc}[1]{\mathcal{#1}}

\newcommand{\diag}{\text{diag}}
\newcommand{\eps}{\epsilon}

\title{A Linearly Convergent GAN Inversion-based Algorithm for Reverse Engineering of Deceptions}

\author{%
  Darshan Thaker \\
  Johns Hopkins University\\
  Baltimore, MD \\
  \texttt{dbthaker@jhu.edu} \\
  \And
  Paris Giampouras \\
  Johns Hopkins University \\
  Baltimore, MD \\
  \texttt{parisg@jhu.edu} \\
  \And
  Ren\'e Vidal \\
  University of Pennsylvania \\
  Philadelphia, PA \\
  \texttt{vidalr@seas.upenn.edu}
}

\begin{document}

\maketitle

\begin{abstract}
An important aspect of developing reliable deep learning systems is devising strategies that make these systems robust to adversarial attacks. There is a long line of work that focuses on developing defenses against these attacks, but recently, researchers have began to study ways to {\it reverse engineer the attack process}. This allows us to not only defend against several attack models, but also classify the threat model. However, there is still a lack of theoretical guarantees for the reverse engineering process. Current approaches that give any guarantees are based on the assumption that the data lies in a union of linear subspaces, which is not a valid assumption for more complex datasets. In this paper, we build on prior work and propose a novel framework for reverse engineering of deceptions which supposes that the clean data lies in the range of a GAN. To classify the signal and attack, we jointly solve a GAN inversion problem and a block-sparse recovery problem.  For the first time in the literature, we provide deterministic {\it linear convergence guarantees} for this problem. We also empirically demonstrate the merits of the proposed approach on several nonlinear datasets as compared to state-of-the-art methods.

\end{abstract}

\section{Introduction}
Modern deep neural network classifiers have been shown to be vulnerable to imperceptible perturbations to the input that can drastically affect the prediction of the classifier. These adversarially attacked inputs can pose problems in safety-critical applications where correct classification is paramount. Adversarial attacks can be either universal perturbations, which remain fixed and can deceive a pretrained network on different images of the same dataset \cite{moosavi2017universal}, or image-dependent perturbations \cite{poursaeed2018generative}. For the latter approach, attack generation for a given classification network entails maximizing a classification loss function subject to various constraints \cite{MadryMSTV18}. For instance, we can assume that the additive perturbation $\delta$ for a clean signal $x$ lies in an $\ell_p$ ball for some $p \geq 1$, i.e., $\delta \in \mathcal{S}_p$, where $\mathcal{S}_p = \{ \delta :\|\delta\|_p \leq 1\}$ \cite{maini2020adversarial}.

Over the last few  years, there has been significant interest in the topic of devising defenses to enhance the adversarial robustness of deep learning systems. Two popular defense strategies are: a) adversarial training-based approaches, \cite{tramer2019adversarial,sinha2018certifying, shafahi2019free}, which rely on a data augmentation strategy and b) adversarial purification-based defenses, which rely on generative models  \cite{samangouei2018defensegan,lee2017generative,yoon2021adversarial}. The latter approach aims to filter out the noisy component of the corrupted image by projecting it onto the image manifold, parametrized by a pretrained deep generative model. 

The constant endeavor to develop reliable deep learning systems has led to a growing interest in methods that adopt a more holistic approach towards adversarial robustness,  known as the {\it Reverse Engineering of Deceptions (RED)} problem. The objective of RED is to go beyond mere defenses by simultaneously {\it defending against the attack} and {\it inferring the deception strategy} followed to corrupt the input, e.g., which $\ell_p$ norm was used to generate the attack \cite{gong2022reverse}. There are various practical methods to reverse engineer adversarial attacks. These works either rely on deep representations of the adversarially corrupted signals that are then used to classify the attacks, \cite{moayeri2021sample} or complicated ad-hoc architectures and black box models \cite{gong2022reverse,goebel2021attribution}. Their effectiveness is only empirically verified, and there is a noticeable lack of theoretical guarantees for the RED problem. 

This inspired the work of \cite{thaker2022reverse}, in which the authors propose the first principled approach for the RED problem. Specifically, for additive $\ell_p$ attacks, they assume that both {\it the signal $x$ and the attack $\delta$ live in unions of linear subspaces} spanned by the blocks of dictionaries $D_s$ and $D_a$ that correspond to the signal and the attack respectively i.e. $x = D_sc_s$ and $\delta = D_a c_a$. These dictionaries are divided into blocks according to the classes of interest for $x$ and $\delta$ (i.e., the signal classification labels for $x$ and the type of $\ell_p$ threat model used for generating $\delta$). The specific form of $D_s$ and $D_a$ gives rise to block-sparse representations for the signal $x$ and the attack $\delta$ with respect to these dictionaries. This motivates their formulation of RED as an inverse optimization problem where the representation vectors $c_s$ and $c_a$ of the clean signal $x$ and attack $\delta$ are learned under a block-sparse promoting regularizer, i.e., 

\begin{equation}
    \min_{c_s,c_a} \| x' - \underbrace{D_s c_s}_{x} - \underbrace{D_a c_a}_{\delta} \|_{2} + \lambda_s \|c_s\|_{1,2} +\lambda_a \|c_a\|_{1,2}. \;\;  \label{eq:linear_model_red}
\end{equation}
Above, $\|\cdot\|_{1,2}$ is a block-sparsity promoting $\ell_1/\ell_2$ norm, \cite{Stojnic:TSP09,Eldar:TIT09, Elhamifar:TSP12}.
To solve this problem, the authors of \cite{thaker2022reverse} use an alternating minimization algorithm for estimating $c_s$ and $c_a$ and accordingly provide theoretical recovery guarantees for the correctness of their approach.

While these recent works undoubtedly demonstrate the importance of the problem of reverse engineering of deceptions (RED), there still exist several challenges.

{\bf Challenges.}  Existing approaches for RED bring to light a common predicament of whether to develop a practically useful method or a less effective, but theoretically grounded one. Specifically, black box model-based approaches for RED \cite{gong2022reverse} are shown to perform well on complex datasets, but lack performance guarantees. Conversely, the approach in \cite{thaker2022reverse} is theoretically sound\footnote{As theoretically shown in \cite{thaker2022reverse}, an $\ell_p$-bounded attack attack on test sample can be reconstructed as a linear combination of attacks on training samples that compose the blocks of the attack dictionary $D_a$.}, but comes with strong assumptions on the data generative model, i.e., that the data live in a union of linear subspaces. It is apparent that this assumption is unrealistic for complex and high-dimensional datasets. Given the limitation of the signal model of \cite{thaker2022reverse}, the main challenge that we aim to address is: 
\begin{tcolorbox}[width=\linewidth, sharp corners=all, colback=white!95!black]
{\it Can we relax the simplistic assumption on the generative model of  the signal made in \cite{thaker2022reverse}, without compromising the theoretical recovery guarantees for solving the RED problem?} 
\end{tcolorbox}

A natural step towards this objective is to {\it leverage the power of deep generative models}, thus building on adversarial purification approaches and suitably adjusting their formulation to the RED problem. However, in doing so, we are left with an inverse problem that is highly non-convex. Namely, the signal reconstruction involves a projection step onto the manifold parameterized by a pretrained deep generative model. Even though this approach is a key ingredient in applications beyond RED, such as adversarial purification, the problem is yet to be theoretically understood. Further, RED involves finding {\it latent representations for both the signal and the attack}. An efficient way to deal with this is to use an {\it alternating minimization algorithm}, as in \cite{thaker2022reverse}. This leads to the following challenge for developing both practical and theoretically grounded algorithms:  
\begin{tcolorbox}[width=\linewidth, sharp corners=all, colback=white!95!black]
{\it Can we provide theoretical guarantees for an alternating minimization algorithm that minimizes a non-convex and non-smooth RED objective?}
\end{tcolorbox}

{\bf Contributions.}
In this work, we propose {\it a novel reverse engineering of deceptions approach} that can be applied to {\it complex datasets} and offers {\it theoretical guarantees}. We address the weakness of the work in \cite{thaker2022reverse} by leveraging the power of nonlinear deep generative models. Specifically, we replace the signal model $x = D_s c_s$ in \eqref{eq:linear_model_red} with $x = G(z)$, where $G: \mathbb{R}^d \rightarrow \mathbb{R}^n, d \ll n$ is the generator of a Generative Adversarial Network (GAN). By using a pre-trained GAN generator, we can reconstruct the clean signal by projecting onto the signal manifold learned by the GAN, i.e., by estimating a $z$ such that $G(z) \approx x$. Further, adversarial perturbations are modeled as in \cite{thaker2022reverse}, i.e., as block-sparse vectors with respect to a predefined dictionary. The inverse problem we solve in this model is then:
\begin{equation}
    \min_{z,c_a} \| x' - \underbrace{G(z)}_{x} - \underbrace{D_a c_a}_{\delta} \|_{2} +\lambda \|c_a\|_{1,2}. \;\;  \label{eq:nonlinear_model_red}
\end{equation}

Our main contributions are the following:

\begin{itemize}[leftmargin=*]

\item {\it A Linearly Convergent GAN inversion-based RED algorithm.} We address the main challenge above and provide recovery guarantees for the signal and attack in two regimes. First, we deal with the unregularized setting, i.e., $\lambda = 0$ in \eqref{eq:nonlinear_model_red}, where we alternate between updating $z$, the latent representation of the estimate of the clean signal, and $c_a$, the attack coefficient, via an alternating gradient descent algorithm. In this setting, we show {\it linear convergence of both iterates jointly to global optima}. Second, as in \cite{thaker2022reverse}, we consider a regularized objective to learn the signal and attack latent variables. In this regime, for an alternating proximal gradient descent algorithm, we show {\it linear convergence in function values to global minima}.  

\item {\it A Linearly Convergent GAN inversion algorithm.} Next, we specialize our results for 
the clean signal reconstruction problem, known as GAN inversion, which is of independent interest. For the GAN inversion problem, we demonstrate linear convergence  of  a subgradient descent algorithm to the global minimizer of the objective function. Note that we rely on assumptions that only require {\it smoothness} of the activation function and a {\it local error-bound} condition. To the best of our knowledge, this is the {\it first result that analyzes the GAN inversion problem departing from the standard assumption of networks with randomized weights \cite{hand2017global,NEURIPS2021_cf77e1f8}}. 

\item {\it SOTA Results for the RED problem.} Finally, we empirically verify our theoretical results on simulated data and also demonstrate new state-of-the-art results for the RED problem using our alternating algorithm on the MNIST, Fashion-MNIST and CIFAR-10 datasets. 
\end{itemize}

\section{Related Work}
{\bf Adversarial Defenses.} We restrict our discussion of adversarial attacks, \cite{carlini_towards_2017, biggio2013evasion}, to the {\it white-box attack} scenario where adversaries have access to the network parameters and craft the attacks usually by solving a loss maximization problem. Adversarial training, a min-max optimization approach, has been the most popular defense strategy \cite{tramer2019adversarial}. Adversarial purification methods are another popular strategy; these methods rely on pretrained deep generative models as a prior for denoising corrupted images \cite{nie2022diffusion, samangouei2018defensegan}. This problem is formulated as an inverse optimization problem \cite{xia2022gan}, and the theoretical understanding of the optimization landscape of the problem is an active area of research \cite{NEURIPS2021_cf77e1f8,hand2017global}. Our work leverages pretrained deep generative models for the RED problem and also aims to shed light on theoretical aspects of the corresponding inverse problems. 

{\bf Theoretical Analysis of GAN-inversion algorithms.}
 In our approach, we employ a GAN-inversion strategy for the RED problem. There is a rich history of deep generative models for inverse problems, such as compressed sensing, \cite{ongie2020deep,jalal2020robust} super-resolution, \cite{Menon_2020_CVPR}, image inpainting, \cite{xia2022gan}. However, efforts to provide theoretical understanding of the landscape of the resulting optimization problem have restricted their attention to the settings where the GAN has random or close-to-random weights \cite{Shah2018,hand2017global,NEURIPS2021_cf77e1f8, lei2019inverting, song2019surfing}. For the first time in the literature, we depart from these assumptions to provide a more holistic analysis of the GAN inversion problem, instead leveraging recent optimization concepts i.e. error-bound conditions and proximal Polyak-{\L}ojasiewicz conditions \cite{karimi2016linear,frei2021proxy,drusvyatskiy2018error}.

{\bf Reverse Engineering of Deceptions (RED).} 
RED is a recent framework to not only defend against attacks, but also reverse engineer and infer the type of attack used to corrupt an input. There are several practical methods proposed for the RED problem. In \cite{goebel2021attribution}, the authors use a multi-class network trained to identify if an image is corrupted and predict attributes of the attack. In \cite{gong2022reverse}, a denoiser-based approach is proposed, where the denoiser weights are learned by aligning the predictions of the denoised input and the clean signal. The authors in \cite{moayeri2021sample} use pretrained self-supervised embeddings e.g. SimCLR \cite{chen2020simple} to classify the attacks. The work most related to ours is \cite{thaker2022reverse}, in which the authors show a provably correct block-sparse optimization approach for RED. Even though \cite{thaker2022reverse} is the first provably correct RED approach, their modelling assumption for the generative model of the clean signal is often violated in real-world datasets. Our work addresses this issue by developing a provable approach with more realistic modelling assumptions.

\section{Problem Formulation}

We build on the formulation of \cite{thaker2022reverse} to develop a model for an adversarial example $x' = x + \delta$, with $x$ being the clean signal and $\delta$ the adversarial perturbation. We replace the signal model of Equation \eqref{eq:linear_model_red} with a pretrained generator $G$ of a GAN. Thus, the generative model we assume for $x$ is given by 
\begin{equation} \label{eq:red_model}
	x' \approx G(z) + D_a c_a.
\end{equation}

We use generators $G:\mathbb{R}^d\rightarrow \mathbb{R}^{n_L}, d\gg n_L$ which are $L$-layer networks of the form 
\begin{equation}
G(z) = \sigma(W_L \sigma(W_{L -1 } \cdots W_2 \sigma (W_1 z)))
\end{equation}
where $W_i \in \R^{n_i \times n_{i - 1}}$ are the known GAN parameters with $n_0 = d$, $\sigma$ is a nonlinear activation function, and $D_a \in \R^{n_L \times k_a}$ is an attack dictionary (typically with $k_a > n_L$). 

As in \cite{thaker2022reverse}, the attack dictionary $D_a$  contains blocks corresponding to different $\ell_p$ attacks (for varying $p$) computed on training samples of each class. The authors of \cite{thaker2022reverse} verify this modelling assumption by showing that for networks that use piecewise linear activations, $\ell_p$ attacks evaluated on test examples can be expressed as linear combinations of $\ell_p$ attacks evaluated on training examples. Using the model in \eqref{eq:red_model}, we then formulate an inverse problem to learn $z$ and $c_a$:
\begin{equation} \label{eq:gen_problem}
    \min_{z, c_a} \mc{L}(z, c_a) \triangleq f(z, c_a) + \lambda h(c_a),
\end{equation}
where $f(z, c_a) = \norm{x' - G(z) - D_a c_a}{2}^2$ denotes a reconstruction loss and $h(c_a)$ denotes a (nonsmooth) convex regularizer on the coefficients $c_a$. For example, in \cite{thaker2022reverse}, the regularizer $h(c_a)$ is $\norm{c_a}{1,2}$ which promotes block-sparsity on $c_a$ according to the structure of $D_a$. We note that our theoretical results do not assume this form for $D_a$, but rather only that its spectrum can be bounded. 

A natural algorithm to learn both $z$ and $c_a$ is to alternate between updating $z$ via subgradient descent and $c_a$ via proximal gradient descent, as shown in Algorithm \ref{alg:red}.


\begin{algorithm}[htbp]
\caption{Proposed RED Algorithm}
\label{alg:red}
Given: $x' \in \R^{n_L}, G: \R^d \to \R^{n_L}, D_a \in \R^{n_L \times k_a}$ \\
Initialize: $z^0, c_a^0$ \\
Set: Step size $\eta$ and regularization parameter $\lambda$
\begin{algorithmic}
	\For{$k = 0, 1, 2, \dots$}
            \State $R_i \gets \text{diag}(\sigma'(W_i z^k))$ for $i \in \{1, \dots, L\}$
            \State $z^{k+1} \gets z^k - \eta (W_1 R_1)^T (W_2 R_2)^T \cdots (W_L R_L)^T (G(z^k) + D_a c_a^k - x')$ 
            \State $c_a^{k+1} \gets \text{prox}_{\lambda h}\left\{c_a^k - \eta D_a^T (G(z^k) + D_a c_a^k - x')\right\}$
	\EndFor
        \State \Return $z^{k+1}, c_a^{k+1}$
\end{algorithmic}
\end{algorithm}

\section{Main Results: Theoretical Guarantees for RED}
In this section, we provide our main theoretical results for the RED problem with a deep generative model used for the clean data. We demonstrate the convergence of the iterates of Algorithm \ref{alg:red} to global optima. A priori, this is difficult due to the non-convexity of \eqref{eq:gen_problem} introduced by the GAN generator $G(z)$ \cite{hand2017global}. To get around this issue, works such as \cite{hand2017global} and \cite{huang2021provably} make certain assumptions to avoid spurious stationary points. However, these conditions essentially reduce to the GAN having weights that behave as a random network (see Definition \ref{def:wdc} in Appendix). In practice, especially for the RED problem, modelling real data often requires GANs with far-from-random weights, so there is a strong need for theoretical results in this setting. 

We draw inspiration from the theory of deep learning and optimization literature, where several works have analyzed non-convex problems through the lens of Polyak-\L{}ojasiewicz (PL) conditions or assumptions that lead to benign optimization landscapes \cite{karimi2016linear, richards2021stability, liu2022loss}. Our goal is to depart from the randomized analysis of previous GAN inversion works to address the non-convexity of the problem. The main assumption we employ is a {\it local error bound} condition. We conjecture this assumption holds true in practice for two reasons. First, we show that the random network conditions assumed in existing works \cite{hand2017global, huang2021provably} already imply a local error bound condition (see Corollary \ref{cor:wdc_gi}). Moreover, in Section \ref{sec:exp:synth_data}, we give examples of non-random networks that also empirically satisfy the local error-bound condition, showing the generality of our assumption. Secondly, the empirical success of GAN inversion in various applications suggests that the optimization landscape is benign \cite{xia2022gan}. However, for the GAN inversion problem, traditional landscape properties such as a PL condition do not hold globally \footnote{We refer the reader to Section 3 of \cite{liu2022loss} for a simple explanation of this phenomenon.}. Nevertheless, we can use local properties of benign regions of the landscape to analyze convergence\footnote{Note that  similar local conditions to analyze convergence have been used in works analyzing the theory of deep learning, such as \cite{liu2022loss}.}. Our work serves as an initial step to analyze convergence of far-from-random networks, and an important avenue of future work is verifying the local error bound condition theoretically for certain classes of networks.

\subsection{Reverse Engineering of Deceptions Optimization Problem without Regularization} \label{sec:red_unreg}

We first consider the unregularized setting where in Algorithm \ref{alg:red}, we only minimize $f(z, c_a)$, i.e. $\lambda=0$ and $\text{prox}_{\lambda h}(\cdot)$ is the identity function. Suppose there exists a $z^*$ and $c_a^*$ such that $x' = G(z^*) + D_a c_a^*$, so $(z^*, c_a^*)$ are global minimizers of $f(z, c_a)$. Our first set of results will ensure convergence of the iterates $(z^k, c_a^k)$ to $(z^*, c_a^*)$. We will denote $\norm{\Delta z^{k + 1}}{2} \triangleq \norm{z^{k + 1} - z^*}{2}$ and $\norm{\Delta c_a^{k + 1}}{2} \triangleq \norm{c_a^{k + 1} - c_a^*}{2}$. To state our convergence results, we posit some assumptions on $G$ and the iterates of the algorithm.

\begin{assumption} (Activation Function) \label{ass:act}
	We assume that $\sigma$ is twice differentiable and smooth. 
\end{assumption}
Note that standard  activation functions such as the sigmoid or smooth ReLU variants (softplus, GeLU, Swish etc.) satisfy Assumption \ref{ass:act}.

\begin{assumption} (Local Error Bound Condition) \label{ass:eb} 
	For all $z^k$ and $c_a^k$ on the optimization trajectory, suppose that there exists a $\mu > 0$ such that
	
	  \begin{equation}
	\norm{\grad_z f(z^k, c_a^k)}{2}^2 + \norm{\grad_{c_a} f(z^k, c_a^k)}{2}^2 \geq \mu^2 (\norm{\Delta z^k}{2}^2 + \norm{\Delta c_a^k}{2}^2)
    \end{equation}
\end{assumption}

Under these assumptions, our main theorem demonstrates linear convergence of the iterates $z^k$ and $c_a^k$ to the global minimizers $z^*$ and $c_a^*$. 

\begin{theorem} \label{thm:red_unreg} 
    Suppose that Assumption \ref{ass:act} holds for the nonlinear activation function and Assumption \ref{ass:eb} holds with local error bound parameter $\mu$. Let $\rho$ and $-\epsilon$ be the maximum and minimum eigenvalues of the Hessian of the loss. Further, assume that the step size satisfies $\eta \leq \min \left\{ \frac{1}{4 \epsilon}, \frac{3}{2 \rho} \right\}$ and $\eta \in \left( \frac{3\mu^2 - \sqrt{9\mu^4 - 32\mu^2 \rho \epsilon}}{4\mu^2 \rho}, \frac{3\mu^2 + \sqrt{9\mu^4 - 32\mu^2 \rho \epsilon}}{4\mu^2 \rho} \right)$. Lastly, assume that $\mu \gtrsim \sqrt{\rho \epsilon}$. Then, we have that the iterates  converge linearly to the global optimum with the following rate in $(0, 1)$:
    \begin{equation}
        \norm{\Delta z^{k + 1}}{2}^2 + \norm{\Delta c_a^{k + 1}}{2}^2 \leq \left(1 - 4 \eta^2 \mu^2 \left( \frac{3}{4} - \frac{\eta \rho}{2} \right) + 4 \eta \epsilon \right) (\norm{\Delta z^k}{2}^2 + \norm{\Delta c_a^k}{2}^2)
    \end{equation}
\end{theorem}

The proof is deferred to the Appendix. Assumption \ref{ass:act} is crucial to our proof, since we show an almost co-coercivity of the gradient (Lemma \ref{lem:coco} in Appendix) that depends on bounding $\rho$ and $\eps$ for smooth and twice differentiable activation functions, similar to the proof strategy of \cite{richards2021stability}. 
 
Along with the step size $\eta$, there are three problem-specific parameters that affect the convergence rate: the largest and the smallest eigenvalues of the Hessian of the loss, i.e., $\rho$ and $-\epsilon$ respectively, and the local error bound parameter $\mu$. Note that because the problem is non-convex, the Hessian will have at least one negative eigenvalue. The rate becomes closer to $1$ and convergence slows as $\epsilon$ gets larger because $\epsilon$ controls the slack in co-coercivity of the gradient in our proof. Similarly, if the operator norm of the weights is controlled, then the convergence rate is faster as a function of $\rho$. Finally, the convergence rate speeds up as $\mu$ increases since each gradient descent iterate takes a larger step towards the minimizer. The condition $\mu \gtrsim \sqrt{\rho \epsilon}$ ensures that the gradient norm is roughly larger than the negative curvature of the Hessian, so that progress towards the global minimizer can still be maintained. The quantity $\sqrt{\rho \eps}$ is the geometric mean of the largest and smallest eigenvalue of the Hessian and can be thought of as a quantity capturing the range of the spectrum of the Hessian. 

Note that extending our results for non-smooth activation functions such as ReLU is nontrivial since we will need to control $\epsilon$. Moreover, due to the almost co-coercivity property of the gradient operator (see Lemma \ref{lem:coco}, Appendix), the step size of gradient descent needs to be bounded away from zero. However, for practical purposes, the regime that is most useful for ensuring fast convergence is when the step size is indeed sufficiently large.

\subsection{Regularized Reverse Engineering of Deceptions Optimization Problem} \label{sec:red_reg} 

We now consider the regularized problem, with $\lambda \neq 0$. The analysis presented in Section \ref{sec:red_unreg} does not immediately extend to this setting because $(z^*, c_a^*)$ now denote minimizers of $\mc{L}(z, c_a) = f(z, c_a) + \lambda h(c_a) $, which is not necessarily the pair $(z^*, c_a^*)$ such that $x' = G(z^*) + D_a c_a^*$. In order to demonstrate convergence, we appeal to well-known results that use the Polyak-\L{}ojasiewicz (PL) condition. We assume a local proximal PL condition on the iterates $c_a^k$, which can be thought of as a version of Assumption \ref{ass:eb} but on the function values instead of the iterates \cite{karimi2016linear}. This assumption also takes into account the proximal update step for $c_a$ \footnote{We refer the reader to \cite{karimi2016linear} for intuition on the global proximal PL inequality}.

\begin{assumption} \label{ass:prox_pl} 
    Let $\rho$ denote the Lipschitz constant of the gradient of $f$ with respect to both $z$ and $c_a$. For all $z^k$ and $c_a^k$ on the optimization trajectory, suppose that there exists a $\mu > 0$ such that
    \begin{equation}
        2\rho \mc{D}(c_a^k, \rho) + \norm{\grad_z f(z^k, c_a^k)}{2}^2 \geq \mu (\mc{L}(z^k, c_a^k) - \mc{L}(z^*, c_a^*)) 
    \end{equation}
    where $\mc{D}(c_a^k, \rho) = - \min_y \left[ \innerprod{\grad_{c_a} f(z^k, c_a^k)}{y - c_a^k} + \frac{\rho}{2} \norm{y - c_a^k}{2}^2 + h(y) - h(c_a^k) \right]$
\end{assumption}

\begin{theorem} \label{thm:red_reg}
    Suppose Assumption \ref{ass:prox_pl} holds with constant $\mu > 0$. Let $\rho$ be the maximum eigenvalue of the Hessian of the loss. If $h$ is convex and $\eta = \frac{1}{\rho}$, then the function values converge linearly:
    \begin{equation}
        \mc{L}(z^{k + 1}, c_a^{k + 1}) - \mc{L}(z^*, c_a^*) \leq \left(1 - \frac{\mu}{2\rho} \right) (\mc{L}(z^k, c_a^k) - \mc{L}(z^*, c_a^*))   
    \end{equation}
\end{theorem}

The proof of this result is in the Appendix, but we note the proof is similar to \cite{karimi2016linear}, Theorem 5. 

\section{Convergence Analysis of the GAN Inversion Problem}

As a special case when there is no adversarial noise, our results also give us a convergence analysis for the realizable GAN inversion problem. This simply corresponds to finding the latent code $z$ for an input $x$ and fixed GAN $G$ such that $G(z) = x$. We let $f(z) \triangleq \norm{x - G(z)}{2}^2$. The following theorem is a specialization of Theorem \ref{thm:red_unreg} to the GAN inversion problem.   

\begin{theorem} \label{thm:gi}
    Suppose that Assumption \ref{ass:act} holds. Further, assume a local error bound condition on the optimization trajectory of $z_k$ with $\mu > 0$:

    \begin{equation}
        \norm{\grad_z f(z^k)}{2} \geq \mu \norm{\Delta z^k}{2}
    \end{equation}
    
    Let $\rho$ and $-\epsilon$ be the maximum and minimum eigenvalues of the Hessian of the loss. Under the same assumptions on the step size $\eta$ and the local error bound parameter $\mu$ as Theorem \ref{thm:red_unreg}, we have that the iterates linearly converge to the global optimum with the following rate in $(0, 1)$:
    \begin{equation}
        \norm{\Delta z^{k + 1}}{2}^2 \leq \left(1 - 4 \eta^2 \mu^2 \left( \frac{3}{4} - \frac{\eta \rho}{2} \right) + 4 \eta \epsilon \right) \norm{\Delta z^k}{2}^2
    \end{equation}
\end{theorem}

The proof of this theorem is identical to the proof of Theorem \ref{thm:red_unreg} by taking $c_a^k = c_a^* = 0$. 

\subsection{Comparison to Existing Approaches}

The works of \cite{hand2017global} and \cite{huang2021provably} derive a condition on the weights of the GAN, which they call the Weight Distribution Condition (WDC), under which they can characterize the optimization landscape of the GAN inversion problem. The WDC ensures the weights of the network behave as close to random networks (see Definition \ref{def:wdc} in Appendix). The authors of \cite{hand2017global} show that under the WDC, there is only one spurious stationary point and the basin of attraction to that point is a small region. The following corollary provides a different viewpoint on this observation by demonstrating that the WDC implies a local error bound condition with parameter $\mu$. This allows us to show a GAN inversion convergence result for subgradient descent.  

\begin{corollary} (GAN Inversion for Networks that satisfy WDC) \label{cor:wdc_gi}
    Let $\epsilon$ be fixed such that $K_1 L^8 \epsilon^{1/4} \leq 1$, where $L$ is the number of the layers of the GAN generator and $K_1$ an absolute constant. Suppose that for all $i \in [L]$, $W_i$ satisfies the WDC with parameter $\epsilon$. Suppose we initialize the iterates $z^0$ of Algorithm \ref{alg:red} that satisfy
    \begin{equation}
        z^0 \notin \mc{B}(z^*, K_2 L^3 \epsilon^{1/4} \norm{z^*}{2}) \cup \mc{B}(-\kappa z^*, K_2 L^{13} \epsilon^{1/4} \norm{z^*}{2}) \cup \{0\}
    \end{equation}
    where $\mc{B}(c, r)$ denotes an $\ell_2$ ball with center $c$ and radius $r$, $K_2$ denotes an absolute constant and $\kappa \in (0, 1)$. 
    Let $\rho$ and $-\epsilon$ be the maximum and minimum eigenvalues of the Hessian of the loss. 
    Then, there exists  $\mu > 0$ such that the local error bound condition holds. Under the same assumptions as Theorem \ref{thm:gi}, we also have that subgradient descent converges linearly to the global optimum with rate $\left(1 - 4 \eta^2 \mu^2 \left( \frac{3}{4} - \frac{\eta \rho}{2} \right) + 4 \eta \epsilon \right)$.
\end{corollary} 

\begin{wrapfigure}{r}{0.37\textwidth}
    \vspace{-4mm}
    \includegraphics[width=\linewidth]{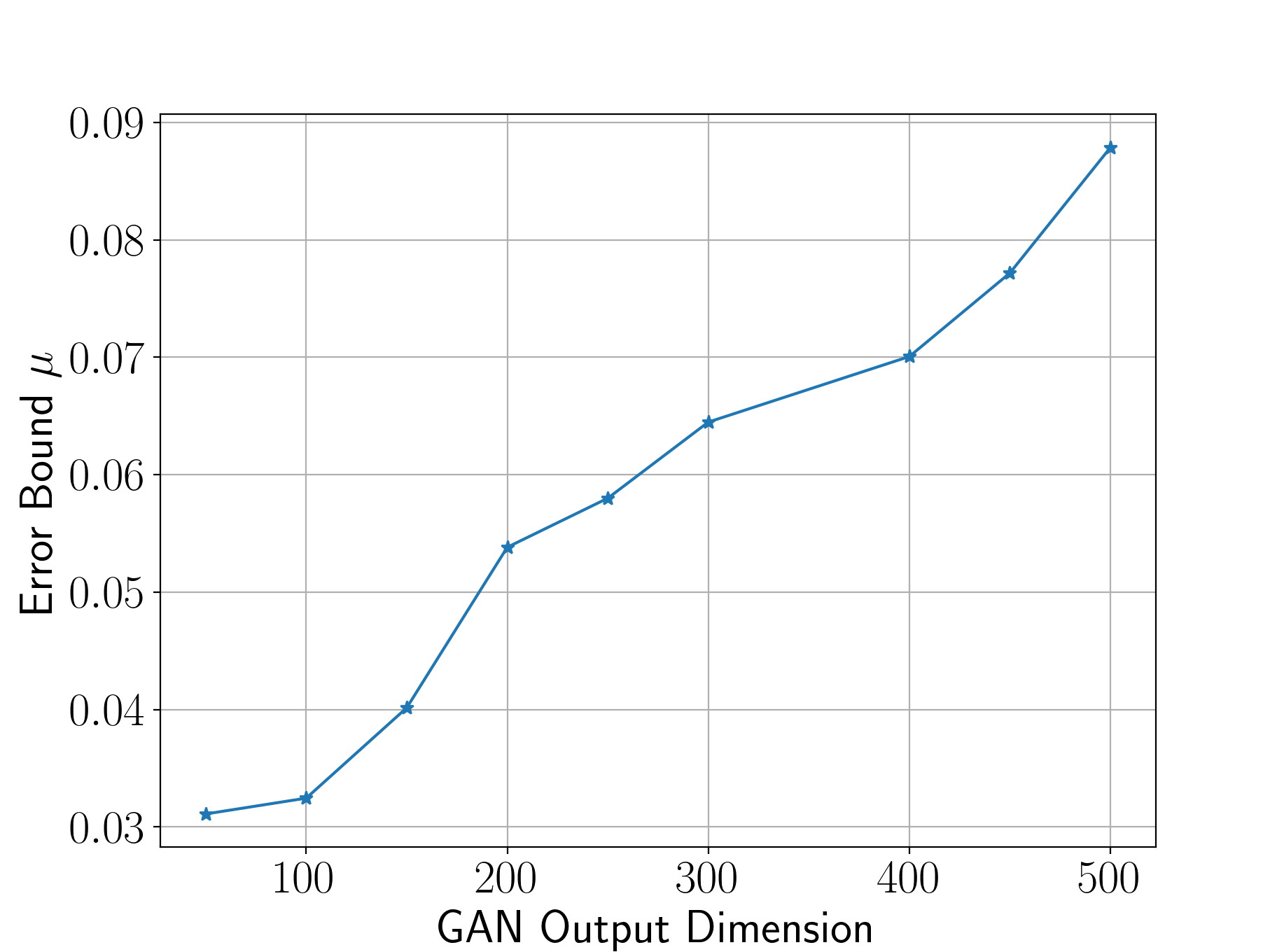}
    \vspace{-4mm}
    \caption{We show the output dimension $m$ vs the computed $\mu$ averaged over the optimization path of 10 test examples for a GAN with random weights and latent space dimension $d = 10$.}
    \vspace{-17mm}
  \label{fig:leb_overpara}
\end{wrapfigure}

To further illustrate the generality of the local error bound condition, we show in Section \ref{sec:exp:synth_data} that the local error bound condition can hold for not only random networks, but also certain classes of non-random networks.

\section{Experiments} \label{sec:exp}

In this section, we provide experiments to verify the local error bound condition, as well as demonstrate the success of our approach on the MNIST, Fashion-MNIST, and CIFAR-10 datasets.

\subsection{Verification of the Local Error Bound Condition}
\label{sec:exp:synth_data}
By studying a realizable RED problem instance, we will demonstrate that the local error bound condition holds for a variety of random and non-random GANs. First, we set up a binary classification task on data $x$ generated from a one-layer GAN $G(z) = \sigma(W z)$ with $W \in \R^{m \times d}$. For a fixed classification network $\psi(x)$, we generate adversarial attacks. Since our problem is realizable, we can compute the error bound parameter $\mu$ exactly. The full experimental setup is given in the Appendix.

\textbf{Random GAN.} We begin by verifying Corollary \ref{cor:wdc_gi} when $W$ is a random matrix. We run our alternating optimization algorithm for $10$ test examples and observe that the iterates always converge to the global optimizer, corroborating our theoretical results.
\begin{wrapfigure}{r}{0.37\textwidth}
    \vspace{-4mm}
    \includegraphics[width=\linewidth]{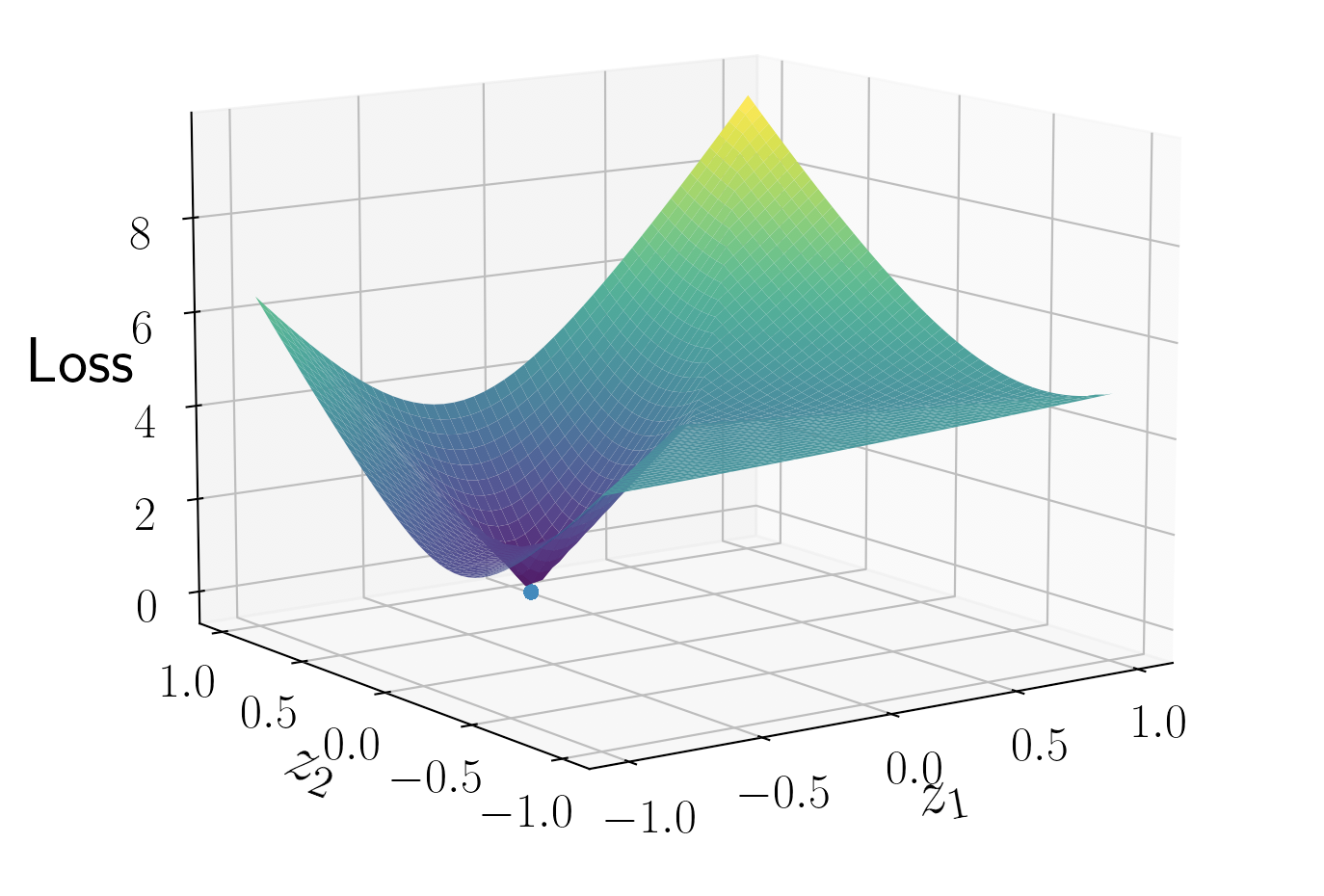}
    \vspace{-4mm}
    \caption{The optimization landscape for a 2-D GAN inversion problem with weights spanned by two orthonormal vectors. See text for details.}
    \vspace{-5mm}
  \label{fig:leb_optim}
\end{wrapfigure}
Moreover, Figure \ref{fig:leb_overpara} shows the effect of the expansiveness of the GAN on the local error bound parameter $\mu$. Many existing results on random GAN inversion assume expansiveness of the GAN ($m \gg d$) to prove a benign optimization landscape. By examining $\mu$ instead, our results offer a different viewpoint. Recall that our convergence theory (Theorem \ref{thm:red_unreg}) shows that as $\mu$ increases, we expect a faster convergence rate. Thus, Figure \ref{fig:leb_overpara} gives further evidence that expansiveness helps optimization and leads to a better landscape.

\textbf{Non-Random GAN.} To illustrate an example of a non-random network that can still satisfy the local error bound condition, consider a GAN with latent space dimension $d = 2$ and output dimension $m = 100$. Suppose that the rows of $W$ are spanned by two orthonormal vectors $\begin{bmatrix} -\sqrt{2}/2 & \sqrt{2}/2 \end{bmatrix} $ and $\begin{bmatrix} \sqrt{2}/2 & \sqrt{2}/2 \end{bmatrix} $. The distribution of these rows is far from the uniform distribution on the unit sphere, and also does not satisfy the Weight Distribution Condition (WDC) from Corollary \ref{cor:wdc_gi} for small values of $\epsilon$ (in \cite{huang2021provably}, $\epsilon$ must be less than $\frac{1}{d^{90}}$ which is a very small number even for $d = 2$). However, the optimization landscape is still benign, and we can reliably converge to the global optimum. For this problem, with a initialization of $z$ as a standard normal random variable and $c_a$ initialized to the all-zeros vector, we observe an average $\mu$ value of $0.013$ over different random initializations. Since $d = 2$, we can plot the landscape for the GAN inversion problem when we set $c_a^* = c_a^k = 0$ - this is shown in Figure \ref{fig:leb_optim} and confirms the benign landscape. Examples of more non-random networks and corresponding values of $\mu$ can be found in the Appendix. 

\subsection{Reverse Engineering of Deceptions on Real Data} 

\textbf{Experimental Setup.} We consider the family of $\{ \ell_1, \ell_2, \ell_\infty\}$ PGD attacks - the full experimental details of the attacks and network architectures can be found in the Appendix. We use a pretrained DCGAN, Wasserstein-GAN, and StyleGAN-XL for the MNIST, Fashion-MNIST and CIFAR-10 datasets respectively \cite{radford2015unsupervised, arjovsky2017wasserstein, sauer2022stylegan, lecun1998mnist, xiao2017fashion, krizhevsky2009learning}. The attack dictionary $D_a$ contains $\ell_p$ attacks for $p \in \{1, 2, \infty\}$ evaluated on $200$ training examples per class. It is divided into blocks where each block corresponds to a signal class and attack type pair, i.e., block $(i, j)$ of $D_a$ denotes signal class $i$ and $\ell_p$ attack type~$j$. 

\textbf{Signal Classification Baselines.} We consider a variety of baselines for the signal classification task. It is important to note that the main task in the RED problem is not to develop a better defense, but rather correctly classify the threat model in a principled manner. Despite this, we compare to various adversarial training mechanisms designed to defend against a union of threat models. The first baselines are $M_1, M_2$ and $M_\infty$, which are adversarial training algorithms for $\ell_1, \ell_2$ and $\ell_\infty$ attacks respectively. We then compare to the SOTA specialized adversarial training algorithm, known as MSD \cite{maini2020adversarial,tramer2019adversarial}. Lastly, we compare to the structured block-sparse classifier (SBSC) from \cite{thaker2022reverse}, which relies on a union of linear subspaces assumption on the data. 

\textbf{Attack Classification Baselines.} Even though the RED problem is understudied, the approach most related to our work is \cite{thaker2022reverse}, which is denoted as the structured block-sparse attack detector (SBSAD). 

\textbf{Algorithm.} To jointly classify the signal and attack for an adversarial example $x'$ computed on classification network $\psi$, we run Algorithm \ref{alg:red}. We initialize $z^*$ to the solution of the Defense-GAN method applied to $x'$, which runs GAN inversion on $x'$ directly \cite{samangouei2018defensegan}. Our methods are:

\begin{enumerate}[leftmargin=*]
    \item BSD-GAN (Block-Sparse Defense GAN): The signal classifier that runs Algorithm \ref{alg:red} and then uses $G(z^k)$ as input to the classification network $\psi$ to generate a label.
    \item BSD-GAN-AD (Block-Sparse Defense GAN Attack Detector): This method returns the block $\hat{j}$ of the attack dictionary $D_a$ that minimizes the reconstruction loss $\norm{x' - G(z^k) - D_a[i][\hat{j}]c_a[i][\hat{j}]}{2}$ for all $i$. 
\end{enumerate}

Further experimental details such as step sizes and initialization details can be found in the Appendix.  

\begin{table*}[h]
\centering
\caption{Adversarial image and attack classification accuracy on digit classification of MNIST dataset. SBSC denotes the structured block-sparse signal classifier and SBSAD denotes the structured block-sparse attack detector. BSD-GAN and BSD-GAN-AD are the Block-Sparse Defense GAN and Block-Sparse Defense GAN Attack Detector respectively.}
\label{table:mnist_def}
\resizebox{\textwidth}{!}{
\begin{tabular}{c||c@{\;\;}c@{\;\;}c@{\;\;}c@{\;\;}c||c@{\;\;}c||@{\;\;}cc} 
 \toprule
\textbf{MNIST} & CNN & $M_\infty$ & $M_2$ & $M_1$ & MSD & SBSC & BSD-GAN & SBSAD & BSD-GAN-AD  \\
 \midrule
 Clean accuracy & 98.99\% & 99.1\% & 99.2\% & 99.0\% & 98.3\% & 92\% & 94\% & - & - \\
 $\ell_\infty$ PGD ($\epsilon = 0.3$) & 0.03\% & \textbf{90.3\%} & 0.4\% & 0.0\% & 62.7\%  & 77.27\% & 75.3\% & 73.2\% & \textbf{92.3\%} \\
 $\ell_2$ PGD ($\epsilon = 2.0$) & 44.13\% & 68.8\% & 69.2\% & 38.7\% & 70.2\% & 85.34\% & \textbf{89.6\%} & 46\% & \textbf{63\%} \\
 $\ell_1$ PGD ($\epsilon = 10.0$) & 41.98\% & 61.8\% & 51.1\% & 74.6\% & 70.4\% & 85.97\% & \textbf{87.8\%} & 36.6\% & \textbf{95.8\%} \\
 \midrule
 Average & 28.71\% & 73.63\% & 40.23\% & 37.77\% & 67.76\% & 82.82\% & \textbf{84.23\%} & 51.93\% & \textbf{83.7\%}  \\
 \bottomrule
\end{tabular}}
\end{table*}

\subsubsection{MNIST and Fashion-MNIST}

For both the MNIST and Fashion-MNIST datasets, we expect that the method from \cite{thaker2022reverse} will not work well since the data does not lie in a union of linear subspaces. Table \ref{table:mnist_def} and \ref{table:fmnist_def} show the signal and attack classification results for the two datasets. Surprisingly, even for the MNIST dataset, the baselines from \cite{thaker2022reverse} are better than the adversarial training baselines at signal classification and it is also able to successfully classify the attack on average. However, our approach improves upon this method since the GAN is a better model of the clean data distribution. The improved data model results in not only higher signal classification accuracy on average, but also significantly higher attack classification accuracy since the signal error is lower. We also observe that for attack classification, discerning between $\ell_2$ and $\ell_1$ attacks is difficult, which is a phenomenon consistent with other works on the RED problem \cite{moayeri2021sample, thaker2022reverse}. 

\begin{table*}[h]
\centering
\caption{Adversarial image and attack classification accuracy on Fashion-MNIST dataset. See Table \ref{table:mnist_def} for column descriptions. }
\begin{tabular}{c||ccc||cc}
 \toprule
  \textbf{Fashion-MNIST} & CNN & SBSC & BSD-GAN & SBSAD & BSD-GAN-AD \\
 \midrule
 $\ell_\infty$ PGD ($\epsilon = 0.3$) & 2\% & 16\% & \textbf{63\%} & 30\% & \textbf{42\%}  \\
 $\ell_2$ PGD ($\epsilon = 2.0$) & 10\% & 20\% & \textbf{68\%} & 55\% & \textbf{59\%}  \\
 $\ell_1$ PGD ($\epsilon = 10.0$) & 12\% & 35\% & \textbf{68\%} & 15\% & \textbf{48\%} \\
 \midrule
 Average & 8\% & 23.67\% & \textbf{66.33\%} & 33.33\% & \textbf{49.66\%}  \\
 \bottomrule
\end{tabular}
\label{table:fmnist_def}
\end{table*}

\subsubsection{CIFAR-10} 

We use a class-conditional StyleGAN-XL to model the clean CIFAR-10 data and a WideResnet as the classification network, which achieves 96\% clean test accuracy. As many works have observed the ease of inverting StyleGANs in the $\mc{W}+$ space (the space generated after the mapping network), we invert in this space \cite{abdal2019image2stylegan}. We initialize the iterates of the GAN inversion problem to a vector in $\mc{W}+$ that is generated by the mapping network applied to a random $z$ and a random class. Interestingly, the GAN inversion problem usually converges to an image of the correct class regardless of the class of the initialization, suggesting a benign landscape of the class-conditional StyleGAN. 

Our results in Table \ref{table:cifar_def} show a $~60\%$ improvement in signal classification accuracy on CIFAR-10 using the GAN model as opposed to the model from \cite{thaker2022reverse}. The attack classification accuracy also improves on average from 37\% to 56\% compared to the model that uses the linear subspace assumption for the data. However, for $\ell_\infty$ and $\ell_1$ attacks, we do not observe very high attack classification accuracy. We conjecture that this is due to the complexity of the underlying classification network, which is a WideResnet \cite{zagoruyko2016wide}. Namely, the results of \cite{thaker2022reverse} show that the attack dictionary model is valid only for fully connected locally linear networks. Extending the attack model to handle a wider class of networks is an important future direction.

\begin{table*}[h!]
\centering
\caption{Adversarial image and attack classification accuracy on CIFAR-10 dataset for 100 test examples. See Table \ref{table:mnist_def} for column descriptions. }
\begin{tabular}{c||ccc||cc}
 \toprule
  \textbf{CIFAR-10} & CNN & SBSC & BSD-GAN & SBSAD & BSD-GAN-AD \\
 \midrule
 $\ell_\infty$ PGD ($\epsilon = 0.03$) & 0\% & 15\% & \textbf{76\%} & 14\% & \textbf{48\%}  \\
 $\ell_2$ PGD ($\epsilon = 0.05$) & 0\% & 18\% & \textbf{87\%} & 36\% & \textbf{77\%}  \\
 $\ell_1$ PGD ($\epsilon = 12.0$) & 0\% & 18\% & \textbf{71\%} & \textbf{63\%} & 44\% \\
 \midrule
 Average & 0\% & 17\% & \textbf{78\%} & 37.66\% & \textbf{56\%}  \\
 \bottomrule
\end{tabular}
\label{table:cifar_def}
\end{table*}

\section{Conclusion}

In this paper, we proposed a GAN inversion-based approach to reverse engineering adversarial attacks with provable guarantees. In particular, we relax assumptions in prior work that clean data lies in a union of linear subspaces to instead leverage the power of nonlinear deep generative models to model the data distribution. For the corresponding nonconvex inverse problem, under local error bound conditions, we demonstrated linear convergence to global optima. Finally, we empirically demonstrated the strength of our model on the MNIST, Fashion-MNIST, and CIFAR-10 datasets. We believe our work has many promising future directions such as verifying the local error bound conditions theoretically as well as relaxing them further to understand the benign optimization landscape of inverting deep generative models.

\bibliography{adversarial,learning,recognition,sparse,vidal,vision,main}
\bibliographystyle{plain}


\clearpage
\appendix

\section*{Supplementary Material for "A Linearly Convergent GAN Inversion-based Algorithm for Reverse Engineering of Deceptions"}
\section{Proofs for Theoretical Results}

\subsection{Proofs for Section \ref{sec:red_unreg}: Reverse Engineering of Deceptions Optimization Problem without Regularization} 

Recall that we formulate an inverse problem to learn $z$ and $c_a$:
\begin{equation} \label{eq:gen_problem2}
    \min_{z, c_a} \mc{L}(z, c_a) \triangleq f(z, c_a) + \lambda h(c_a),
\end{equation}
where $f(z, c_a) = \norm{x' - G(z) - D_a c_a}{2}^2$ denotes a reconstruction loss and $h(c_a)$ denotes a (nonsmooth) convex regularizer on the coefficients $c_a$.

The proof strategy for our main theorem in Section \ref{sec:red_unreg} relies mainly on an almost co-coercivity of the gradient (Lemma \ref{lem:coco}), which we show next. 

\begin{lemma} \label{lem:coco} (Almost co-coercivity)
    We have that the gradient operator of  $f(z, c_a)$ is almost co-coercive i.e.
    \begin{align}
		& \innerprod{\grad_{c_a} f(z^k, c_a^k) - \grad_{c_a} f(z^*, c_a^*)}{c_a^k - c_a^*} + \innerprod{\grad_z f(z^k, c_a^k) - \grad_z f(z^*, c_a^*)}{z^k - z^*} \geq \\
		\nonumber & 2\eta \left( 1 - \frac{\eta \rho}{2} \right) \left[ \norm{ \grad_z f(z^k, c_a^k) - \grad_z f(z^*, c_a^*)}{2}^2 + \norm{\grad_{c_a} f(z^k, c_a^k) - \grad_{c_a} f(z^*, c_a^*)}{2}^2 \right] \\
		 \nonumber & - \epsilon [ \norm{z^k - z^* - \eta(\grad_z f(z^k, c_a^k) - \grad_z f(z^*, c_a^*))}{2}^2 \\
           \nonumber & + \norm{c_a^k - c_a^* - \eta (\grad_{c_a} f(z^k, c_a^k) - \grad_{c_a} f(z^*, c_a^*))}{2}^2 ]
	\end{align}

        where $\rho$ and $-\epsilon$ denote the maximum and minimum eigenvalues of the Hessian of the loss respectively. 
\end{lemma}

\begin{proof}
We note that the proof of this result is adapted from \cite{richards2021stability}, Lemma 5. The key differences are that we  do not consider a stability of iterates when changing one datapoint as in \cite{richards2021stability}, but rather show an almost co-coercivity of the gradient across iterates of our gradient descent algorithm. Further, our analysis requires extra assumptions such as the local error bound condition in order to demonstrate convergence of the iterates beyond this lemma. 

We wish to lower bound $\innerprod{\grad_{c_a} f(z^k, c_a^k) - \grad_{c_a} f(z^*, c_a^*)}{c_a^k - c_a^*} + \innerprod{\grad_z f(z^k, c_a^k) - \grad_z f(z^*, c_a^*)}{z^k - z^*}$. We can rewrite this inner product in a different way using the functions:

\begin{align}
	\psi(z, c_a) &\triangleq f(z, c_a) - \innerprod{\grad_z f(z^*, c_a^*)}{z} - \innerprod{\grad_{c_a} f(z^*, c_a^*)}{c_a} \\
	\psi^\star(z, c_a) &\triangleq f(z, c_a) - \innerprod{\grad_z f(z^k, c_a^k)}{z} - \innerprod{\grad_{c_a} f(z^k, c_a^k)}{c_a}
\end{align}

Then, some simple algebra shows that:

\begin{align}
	&\innerprod{\grad_{c_a} f(z^k, c_a^k) - \grad_{c_a} f(z^*, c_a^*)}{c_a^k - c_a^*} + \innerprod{\grad_z f(z^k, c_a^k) - \grad_z f(z^*, c_a^*)}{z^k - z^*} \\
	&\qquad \qquad = \psi(z^k, c_a^k) - \psi(z^*, c_a^*) + \psi^\star(z^*, c_a^*) - \psi^\star(z^k, c_a^k) 
\end{align}

Now, we will bound $ \psi(z^k, c_a^k) - \psi(z^*, c_a^*)$ and $\psi^\star(z^*, c_a^*) - \psi^\star(z^k, c_a^k)$ separately. 

We prove it for $\psi(z^k, c_a^k) - \psi(z^*, c_a^*)$ and the proofs are symmetric replacing $\psi$ with $\psi^\star$. The proof strategy will be to upper and lower bound a different term, namely $\psi(z^k - \eta \grad_z \psi(z^k, c_a^k), c_a^k - \eta \grad_{c_a} \psi(z^k, c_a^k))$. 

We begin with the upper bound, which uses $\rho$-smoothness of the loss and Taylor's approximation to give:

\begin{align}
	\psi(z^k - \eta \grad_z \psi(z^k, c_a^k), c_a^k &- \eta \grad_{c_a} \psi(z^k, c_a^k)) \leq \psi(z^k, c_a^k) \\
& \qquad - \eta \left( 1 - \frac{\eta \rho}{2} \right) \left( \norm{\grad_z \psi(z^k, c_a^k)}{2}^2 + \norm{\grad_{c_a} \psi(z^k, c_a^k)}{2}^2 \right)
\end{align}

The lower bound is a bit more tricky (we normally would just use convexity and smoothness to lower bound by $\psi(z^k, c_a^k)$). We start by defining the following quantities:

\begin{align}
	z(\alpha) &= z^* + \alpha (z^k - z^* - \eta (\grad_z f(z^k, c_a^k) - \grad_z f(z^*, c_a^*)) \\
	c_a(\alpha) &= c_a^* + \alpha (c_a^k - c_a^* - \eta (\grad_{c_a} f(z^k, c_a^k) - \grad_{c_a} f(z^*, c_a^*)
\end{align}

We then define a function $g(\alpha)$ as:

\begin{equation}
	g(\alpha) = \psi(z(\alpha), c_a(\alpha)) + \frac{\epsilon \alpha^2}{2} ( \norm{\zeta_z}{2}^2 + \norm{\zeta_{c_a}}{2}^2 )
\end{equation}

where $\zeta_z = z^k - z^* - \eta (\grad_z f(z^k, c_a^k) - \grad_z f(z^*, c_a^*))$ and $\zeta_{c_a} = c_a^k - c_a^* - \eta (\grad_{c_a} f(z^k, c_a^k) - \grad_{c_a} f(z^*, c_a^*))$ Now, we have that $g''(\alpha) \geq 0$ since $-\epsilon$ is the smallest eigenvalue of the Hessian of the loss, and using the following expansion of $g''(\alpha)$:

\begin{align}
	g'(\alpha) &= \begin{bmatrix} \zeta_z & \zeta_{c_a} \end{bmatrix}^T (\grad f(z(\alpha), c_a(\alpha)) - \grad_z f(z^*, c_a^*) - \grad_{c_a} f(z^*, c_a^*))+ \epsilon \alpha (\norm{\zeta_z}{2}^2 + \norm{\zeta_{c_a}}{2}^2) \\
	g''(\alpha) &= \begin{bmatrix} \zeta_z & \zeta_{c_a} \end{bmatrix}^T \grad^2 f(z(\alpha), c_a(\alpha)) \begin{bmatrix} \zeta_z & \zeta_{c_a} \end{bmatrix} + \epsilon (\norm{\zeta_z}{2}^2 + \norm{\zeta_{c_a}}{2}^2)
\end{align}

This shows that $g$ is convex, so we have that $g(1) - g(0) \geq g'(0) = 0$. Plugging in the definition of $g(1)$ and $g(0)$, we have: 

\begin{align}
	\psi(z^k - \eta \grad_z \psi(z^k, c_a^k), c_a^k - \eta \grad_{c_a} \psi(z^k, c_a^k)) &\geq \psi(z^*, c_a^*) - \frac{\epsilon}{2} \left( \norm{\zeta_z}{2}^2 + \norm{\zeta_{c_a}}{2}^2 \right) 
\end{align}

Rearranging the above lower and upper bounds, we are left with a lower bound on $ \psi(z^k, c_a^k) - \psi(z^*, c_a^*)$ as desired. As mentioned previously, the same arguments hold above to give a lower bound on $\psi^\star(z^*, c_a^*) - \psi^\star(z^k, c_a^k)$. This gives us our final bound, which is that:

\begin{align}
		& \innerprod{\grad_{c_a} f(z^k, c_a^k) - \grad_{c_a} f(z^*, c_a^*)}{c_a^k - c_a^*} + \innerprod{\grad_z f(z^k, c_a^k) - \grad_z f(z^*, c_a^*)}{z^k - z^*} \geq \\
		 \nonumber & 2\eta \left( 1 - \frac{\eta \rho}{2} \right) \left[ \norm{ \grad_z f(z^k, c_a^k) - \grad_z f(z^*, c_a^*)}{2}^2 + \norm{\grad_{c_a} f(z^k, c_a^k) - \grad_{c_a} f(z^*, c_a^*)}{2}^2 \right] \\
		 \nonumber & - \epsilon [ \norm{z^k - z^* - \eta(\grad_z f(z^k, c_a^k) - \grad_z f(z^*, c_a^*))}{2}^2 \\
         \nonumber &+ \norm{c_a^k - c_a^* - \eta (\grad_{c_a} f(z^k, c_a^k) - \grad_{c_a} f(z^*, c_a^*))}{2}^2 ]
\end{align}
\end{proof}

Another important lemma before proving our main result is the following.

\begin{lemma} \label{lem:zeta} 
	Define the following quantities: 
 
    \begin{align} 
        \zeta_z &\triangleq z^k - z^* - \eta (\grad_z f(z^k, c_a^k) - \grad_z f(z^*, c_a^*)) \\
        \zeta_{c_a} &\triangleq c_a^k - c_a^* - \eta (\grad_{c_a} f(z^k, c_a^k) - \grad_{c_a} f(z^*, c_a^*))
    \end{align} 
    
    Then, assuming that $\eta < \frac{1}{2\epsilon}$ and $\eta < \frac{3}{2 \rho}$ , we have:

	\begin{equation}
		\norm{\zeta_z}{2}^2 + \norm{\zeta_{c_a}}{2}^2 \leq \frac{1}{1 - 2\eta \epsilon}(\norm{\Delta z^k}{2}^2 + \norm{\Delta c_a^k}{2}^2) 
	\end{equation}
	
\end{lemma}

\begin{proof}
	We will expand the definition of $\zeta_z$ and $\zeta_{c_a}$ and use co-coercivity (Lemma \ref{lem:coco}) again. This gives us that $\norm{\zeta_z}{2}^2 + \norm{\zeta_{c_a}}{2}^2$ is equal to: 
	
	\begin{align}
	    &= \norm{\Delta z^k}{2}^2 + \norm{\Delta c_a^k}{2}^2 \\
		& \qquad + \eta^2 \left[  \norm{ \grad_z f(z^k, c_a^k) - \grad_z f(z^*, c_a^*)}{2}^2 + \norm{\grad_{c_a} f(z^k, c_a^k) - \grad_{c_a} f(z^*, c_a^*)}{2}^2 \right] \\
		& \qquad - 2\eta \left[ \innerprod{\grad_{c_a} f(z^k, c_a^k) - \grad_{c_a} f(z^*, c_a^*)}{c_a^k - c_a^*} + \innerprod{\grad_z f(z^k, c_a^k) - \grad_z f(z^*, c_a^*)}{z^k - z^*} \right] \\
		&\leq \norm{\Delta z^k}{2}^2 + \norm{\Delta c_a^k}{2}^2 \\
		& \qquad + \eta^2 \left[  \norm{ \grad_z f(z^k, c_a^k) - \grad_z f(z^*, c_a^*)}{2}^2 + \norm{\grad_{c_a} f(z^k, c_a^k) - \grad_{c_a} f(z^*, c_a^*)}{2}^2 \right] \\
		& \qquad - 4 \eta^2 \left( 1 - \frac{\eta \rho}{2} \right) \left[ \norm{ \grad_z f(z^k, c_a^k) - \grad_z f(z^*, c_a^*)}{2}^2 + \norm{\grad_{c_a} f(z^k, c_a^k) - \grad_{c_a} f(z^*, c_a^*)}{2}^2 \right] \\
		& \qquad + 2 \eta \epsilon \left[ \norm{\zeta_z}{2}^2 + \norm{\zeta_{c_a}}{2}^2 \right] \\
		&\leq \left(1 + \eta^2 \rho^2 - 4 \eta^2 \rho^2  \left( 1 - \frac{\eta \rho}{2} \right) \right) \left[ \norm{\Delta z^k}{2}^2 + \norm{\Delta c_a^k}{2}^2 \right]  + 2 \eta \epsilon \left[ \norm{\zeta_z}{2}^2 + \norm{\zeta_{c_a}}{2}^2 \right]
	\end{align}
	
	If $\eta < \frac{1}{2\epsilon}$, we can rearrange this inequality such that
	
	\begin{equation} 
		\norm{\zeta_z}{2}^2 + \norm{\zeta_{c_a}}{2}^2 \leq \frac{1}{1 - 2\eta \epsilon} \left( 1 + \eta^2 \rho^2 - 4\eta^2 \rho^2 \left( 1 - \frac{\eta \rho}{2} \right) \right) \left[ \norm{\Delta z^k}{2}^2 + \norm{\Delta c_a^k}{2}^2 \right]
	\end{equation}
	
	Lastly, using the assumption that $\eta < \frac{3}{2 \rho}$, we have that $\eta^2 \rho^2 - 4 \eta^2 \rho^2 \left( 1 - \frac{\eta \rho}{2} \right) < 0$, so we can drop it from the equation above and have a simpler upper bound: 
	
	\begin{equation}
		\norm{\zeta_z}{2}^2 + \norm{\zeta_{c_a}}{2}^2 \leq \frac{1}{1 - 2\eta \epsilon}(\norm{\Delta z^k}{2}^2 + \norm{\Delta c_a^k}{2}^2) 
	\end{equation}
	
\end{proof}

Now, we can prove our main result, Theorem \ref{thm:red_unreg} restated below.

\begin{theorem}  
    Suppose that Assumption \ref{ass:act} holds for the nonlinear activation function and Assumption \ref{ass:eb} holds with local error bound parameter $\mu$. Let $\rho$ and $-\epsilon$ be the maximum and minimum eigenvalues of the Hessian of the loss. Further, assume that the step size satisfies $\eta \leq \min \left\{ \frac{1}{4 \epsilon}, \frac{3}{2 \rho} \right\}$ and $\eta \in \left( \frac{3\mu^2 - \sqrt{9\mu^4 - 32\mu^2 \rho \epsilon}}{4\mu^2 \rho}, \frac{3\mu^2 + \sqrt{9\mu^4 - 32\mu^2 \rho \epsilon}}{4\mu^2 \rho} \right)$. Lastly, assume that $\mu \gtrsim \sqrt{\rho \epsilon}$. Then, we have that the iterates  converge linearly to the global optimum with the following rate in $(0, 1)$:
    \begin{equation}
        \norm{\Delta z^{k + 1}}{2}^2 + \norm{\Delta c_a^{k + 1}}{2}^2 \leq \left(1 - 4 \eta^2 \mu^2 \left( \frac{3}{4} - \frac{\eta \rho}{2} \right) + 4 \eta \epsilon \right) (\norm{\Delta z^k}{2}^2 + \norm{\Delta c_a^k}{2}^2)
    \end{equation}
\end{theorem}

\begin{proof}

    We can expand the suboptimality of the iterates $\norm{\Delta z^{k + 1}}{2}^2 + \norm{\Delta c_a^{k + 1}}{2}^2$ at the $k + 1$ iteration of gradient descent as follows:
    
    \begin{align}
	 &= \norm{c_a^k - \eta \grad_{c_a} f(z^k, c_a^k) - c_a^* - \eta \grad_{c_a} f(z^*, c_a^*)}{2}^2 \\
	& \qquad  + \norm{z^k - \eta \grad_z f(z^k, c_a^k) - z^* + \eta \grad_z f(z^*, c_a^*)}{2}^2 \\
	&\leq \norm{c_a^k - c_a^*}{2}^2 - 2 \eta \innerprod{\grad_{c_a} f(z^k, c_a^k) - \grad_{c_a} f(z^*, c_a^*)}{c_a^k - c_a^*} \\
	& \qquad  + \eta^2 \norm{\grad_{c_a} f(z^k, c_a^k) - \grad_{c_a} f(z^*, c_a^*)}{2}^2 \\
	& \qquad  + \norm{z^k - z^*}{2}^2 - 2 \eta \innerprod{\grad_z f(z^k, c_a^k) - \grad_z f(z^*, c_a^*)}{z^k - z^*} \\
	& \qquad  + \eta^2 \norm{\grad_z f(z^k, c_a^k) - \grad_z f(z^*, c_a^*)}{2}^2 \\
	&\stackrel{\text{(Lemma \ref{lem:coco})}}{\leq} \norm{\Delta z^k}{2}^2 + \norm{\Delta c_a^k}{2}^2 \\
	& \qquad   - 4\eta^2 \left( \frac{3}{4} - \frac{\eta \rho}{2} \right) \norm{ \grad_z f(z^k, c_a^k) - \grad_z f(z^*, c_a^*)}{2}^2 \label{eq:grad_diff1} \\
	& \qquad - 4\eta^2 \left( \frac{3}{4} - \frac{\eta \rho}{2} \right) \norm{\grad_{c_a} f(z^k, c_a^k) - \grad_{c_a} f(z^*, c_a^*)}{2}^2 \label{eq:grad_diff2} \\
	& \qquad  + 2 \eta \epsilon \norm{z^k - z^* - \eta(\grad_z f(z^k, c_a^k) - \grad_z f(z^*, c_a^*))}{2}^2 \\
	& \qquad + 2\eta \epsilon \norm{c_a^k - c_a^* - \eta (\grad_{c_a} f(z^k, c_a^k) - \grad_{c_a} f(z^*, c_a^*))}{2}^2 \\
	&\stackrel{\text{(Error Bound)}}{\leq} \norm{\Delta z^k}{2}^2 + \norm{\Delta c_a^k}{2}^2 \\
	& \qquad - 4 \eta^2 \mu^2 \left( \frac{3}{4} - \frac{\eta \rho}{2} \right) (\norm{\Delta z^k}{2}^2 + \norm{\Delta c_a^k}{2}^2 ) \\
	& \qquad  + 2 \eta \epsilon \norm{z^k - z^* - \eta(\grad_z f(z^k, c_a^k) - \grad_z f(z^*, c_a^*))}{2}^2 \\
	& \qquad + 2\eta \epsilon \norm{c_a^k - c_a^* - \eta (\grad_{c_a} f(z^k, c_a^k) - \grad_{c_a} f(z^*, c_a^*))}{2}^2 \\
	&\stackrel{\text{(Lemma \ref{lem:zeta})}}{\leq} \left(1 - 4 \eta^2 \mu^2 \left( \frac{3}{4} - \frac{\eta \rho}{2} \right) + \frac{2\eta \epsilon}{1 - 2\eta \epsilon} \right) (\norm{\Delta z^k}{2}^2 + \norm{\Delta c_a^k}{2}^2)
\end{align}

When $\eta < \frac{1}{4\epsilon}$, the last term in the rate $\frac{2\eta \epsilon}{1 - 2\eta \epsilon}$ is upper bounded by $4 \eta \epsilon$. We now examine when this rate is less than $1$. This is equivalent to showing that

\begin{equation}
        - 4 \eta^2 \mu^2 \left( \frac{3}{4} - \frac{\eta \rho}{2} \right) + 4 \eta \epsilon < 0
\end{equation}

Factoring out a factor of $\eta$, we are left with a quadratic in $\eta$. We need the discriminant of this quadratic to be positive in order to have real roots. This gives us a condition that (dropping constant factors):

\begin{equation}
    \mu \gtrsim \sqrt{\rho \eps} 
\end{equation}

The roots of this quadratic in $\eta$ give us a range where the rate is less than $1$. Thus, we require the the step size to be in this range for convergence:

\begin{equation}
    \eta \in \left( \frac{3\mu^2 - \sqrt{9\mu^4 - 32\mu^2 \rho \epsilon}}{4\mu^2 \rho}, \frac{3\mu^2 + \sqrt{9\mu^4 - 32\mu^2 \rho \epsilon}}{4\mu^2 \rho} \right)
\end{equation}

\end{proof}

\subsection{Proof of Theorem \ref{thm:red_reg}: Regularized Case}

\begin{proof}
    First, we note that because the function $f$ is $\rho$-smooth by assumption, we have that for all $(z, c_a)$ and $(\tilde{z}, \tilde{c_a})$:

    \begin{equation}
        f(z, c_a) \leq f(\tilde{z}, \tilde{c_a}) + \innerprod{\grad f(\tilde{z}, \tilde{c_a})}{\begin{bmatrix} z - \tilde{z} & c_a - \tilde{c_a} \end{bmatrix}} + \frac{\rho}{2} \norm{ \begin{bmatrix} z - \tilde{z} & c_a - \tilde{c_a} \end{bmatrix} }{2}^2
    \end{equation}

    Next, we expand the loss:

    \begin{align}
        \mc{L}(z^{k + 1}, c_a^{k + 1}) &= f(z^{k + 1}, c_a^{k + 1}) + h(c_a^{k + 1}) + h(c_a^k) - h(c_a^k) \\
        &= f(z^k, c_a^k) + h(c_a^k) + \innerprod{\grad f(z^k, c_a^k)}{\begin{bmatrix} z^{k + 1} - z^k & c_a^{k + 1} - c_a^k \end{bmatrix}} \\
        & \nonumber \qquad \qquad + \frac{\rho}{2} \norm{ \begin{bmatrix} z^{k + 1} - z^k & c_a^{k + 1} - c_a^k \end{bmatrix} }{2}^2 + h(c_a^{k + 1}) - h(c_a^k) \\
        &= \mc{L}(z^k, c_a^k) + \min_y \left[ \innerprod{\grad_{c_a} f(z^k, c_a^k)}{y - c_a^k)} + \frac{\rho}{2} \norm{y - c_a^k}{2}^2 + h(y) - h(c_a^k) \right] \\
        & \nonumber \qquad \qquad + \innerprod{\grad_z f(z^k, c_a^k)}{z^{k +1} - z^k} + \frac{\rho}{2} \norm{z^{k + 1} - z^k}{2}^2 \\
        &\leq \mc{L}(z^k, c_a^k) + \min_y \left[ \innerprod{\grad_{c_a} f(z^k, c_a^k)}{y - c_a^k)} + \frac{\rho}{2} \norm{y - c_a^k}{2}^2 + h(y) - h(c_a^k) \right] \\
        & \nonumber \qquad \qquad - \eta \norm{\grad_z f(z^k, c_a^k)}{2}^2 + \frac{\eta^2 \rho}{2} \norm{\grad_z f(z^k, c_a^k)}{2}^2 \\
        &\leq \mc{L}(z^k, c_a^k) - \frac{\mu}{\rho} (\mc{L}(z^k, c_a^k) - \mc{L}(z^*, c_a^*)) 
    \end{align}

    This implies our final result:

    \begin{equation}
        \mc{L}(z^{k + 1}, c_a^{k + 1}) - \mc{L}(z^*, c_a^*) \leq \left(1 - \frac{\mu}{\rho} \right) (\mc{L}(z^k, c_a^k) - \mc{L}(z^*, c_a^*))   
    \end{equation}
\end{proof}

\subsection{Instantiating \texorpdfstring{$\rho$}{rho} 
 and \texorpdfstring{$\epsilon$}{epsilon}  for a Simple Network}

When we assume that the network weights and $D_a$ have bounded spectrum as well as the inputs being bounded, we can derive a bound on $\rho$ and $\epsilon$ for a network. For simplicity, we consider a $1$-layer network although these arguments will generalize to the $L$-layer case as well. Formally, let $G(z) = \sigma(Wz)$. We use the following assumptions to bound $\rho$ and $\eps$:

\begin{assumption} (Loss Function and Weights) \label{ass:loss_weights}
	Assume that $\norm{x' - G(z) - D_a c_a}{1} \leq C_0$ for all $z^k, c_a^k$ for an absolute constant $C_0$ i.e. the $\ell_1$ loss is bounded uniformly. This is equivalent to assuming that the inputs $z$ and $c_a$ are bounded. Assume that $\norm{W_i}{2} \leq C_{W_i}$ for all $i$, $\sum_{j = 1}^{n_i} \norm{W_i[j]}{2}^2 \leq V_{W_i}$\footnote{$W_i[j]$ denotes the $j$th row of $W_i$.}, $\norm{D_a}{2} \leq C_D$ and $\sigma_{\min}(D_a) > 0$. 
\end{assumption}

\begin{lemma}
    Fix $z$ and $c_a$. Suppose that Assumptions \ref{ass:act} and \ref{ass:loss_weights} hold. Then, 

	\begin{align}
		\lambda_{\max} (\grad^2 f(z, c_a)) &\leq \rho \\
		\min_{\alpha \in [0, 1]} \lambda_{\min} (\grad^2 f(z + \alpha(z^* - z), c_a + \alpha(c_a^* - c_a))) &\geq - \epsilon
	\end{align}
	
	with $\rho = C_W^2 (B_{\sigma'}^2 + B_{\sigma''} \sqrt{C_0}) + C_W B_{\sigma'} C_D + C_D^2$ and $\epsilon = V_W B_{\sigma''} C_0 + C_W B_{\sigma'} C_D - L_D^2 $. 
\end{lemma}

\begin{proof}
	We begin with noting the Hessian has a block structure due to the two variables:
	
	\begin{equation}
		\grad^2 f(z, c_a) = \begin{bmatrix}\grad_{z, z} f(z, c_a) & \grad_{z, c_a} f(z, c_a) \\ \grad_{c_a, z} f(z, c_a) & \grad_{c_a, c_a} f(z, c_a) \end{bmatrix}
	\end{equation}
	
	These blocks are equal to:
	
	\begin{align}
		\grad_{z, z} [f(z, c_a)] &= \grad_z G(z) \grad_z G(z)^T + \grad_{z, z} G(z) (x' - G(z) - D_a c_a) \\
		\grad_{z, c_a} f(z, c_a) &= \grad_z G(z) D_a = W^T \sigma'(Wz) D_a \\
		\grad_{c_a, z} f(z, c_a) &= D_a^T \grad_z G(z)^T \\
		\grad_{c_a, c_a} f(z, c_a) &= D_a ^T D_a  
	\end{align}
	
        Above, $\grad_{z, z} G(z)$ is actually a tensor of dimension $d \times d \times m$ since $G$ maps from $\R^d$ to $\R^m$. If we take one slice of this tensor i.e. $\grad_{z, z} [G(z)]_i$, we will be left with $\sigma''(W_i \cdot z) W_i W_i^T$, where $W_i$ denotes the $i$th row of $W$. We can write this tensor-vector product as the following:

        \begin{equation}
            \grad_{z, z} G(z) (x' - G(z) - D_a c_a = \sum_{i = 1}^m \sigma''(W_i \cdot z) W_i W_i^T \cdot [x' - G(z) - D_a c_a]_i 
        \end{equation}
        
        Let $M_{z, z}, M_{z, c_a}, M_{c_a, z}, M_{c_a, c_a}$ be four block matrices corresponding to one nonzero block (conformal to the order in the subscript that gradients are taken) and all the other blocks zero. We can bound the operator norm of the Hessian using triangle inequality:

        \begin{align}
            \norm{\grad^2 f(z, c_a)}{2} &= \norm{ M_{z, z} + M_{z, c_a} + M_{c_a, z} + M_{c_a, c_a}}{2} \\
            &\leq \norm{M_{z, z}}{2} + \norm{M_{z, c_a}}{2} + \norm{M_{c_a, z}}{2} + \norm{M_{c_a, c_a}}{2} \label{eq:rho_tri}
        \end{align}
        
        These blocks have operator norm as follows:
	
	\begin{align}
		\norm{M_{z, z}}{2} &\leq C_W^2 B_{\sigma'}^2 + 2 V_W B_{\sigma''} C_0 \\
		\norm{M_{z, c_a}}{2} &\leq C_W B_{\sigma'} C_D \\
		\norm{M_{c_a, z}}{2} &\leq C_W B_{\sigma'} C_D \\
		\norm{M_{c_a, c_a}}{2} &\leq C_D^2 
	\end{align}
	
	Plugging into \eqref{eq:rho_tri} yields $\rho$. For the minimum eigenvalue, we have by Weyl's inequality for Hermitian matrices that:

        \begin{align}
            \lambda_{\min}(\grad^2 f(z, c_a)) &\geq \lambda_{\min} ( M_{z, z}) + \lambda_{\min} \left( \begin{bmatrix} 0 & \grad_{z, c_a} f(z, c_a) \\ \grad_{c_a, z} f(z, c_a) & 0 \end{bmatrix} \right) + \lambda_{\min}( M_{c_a, c_a}) \\
            &\geq - \norm{M_{z, z}}{2} - (\norm{M_{z, c_a}}{2} + \norm{M_{c_a, z}}{2}) + L_D^2 \\
            &\geq - (C_W^2 ( B_{\sigma'}^2 + 2 B_{\sigma''} \sqrt{C_0}) + 2C_W B_{\sigma'} C_D - L_D^2 )
        \end{align}

        Note that because the first term in $M_{z, z}$ is PSD, we can remove it from the lower bound, which gives the bound for $\epsilon$ in the theorem. 
\end{proof}

\subsection{GAN Inversion}

\begin{definition} (Weight Distribution Condition \cite{hand2017global}) \label{def:wdc} 
    A matrix $W \in \R^{n \times k}$ satisfies the Weight Distribution Condition (WDC) with constant $\epsilon$ if for all nonzero $x, y \in R^k$,

    \begin{equation}
        \norm{\sum_{i = 1}^n \sigma'(w_i \cdot x) \sigma'(w_i \cdot y) \cdot w_i w_i^t - Q_{x, y}}{2} \leq \epsilon
    \end{equation}

    with $Q_{x, y} = \E[\sigma'(w_i \cdot x) \sigma'(w_i \cdot y) \cdot w_i w_i^t]$ for $w_i \sim N(0, I_k/n)$. 
\end{definition}

\subsubsection{Proof of Corollary \ref{cor:wdc_gi}}

\begin{proof}
    From Theorem 2 in \cite{hand2017global}, we have that when the conditions of the corollary are met, then there exists a direction $v$ such that the directional derivative of $f(z^0)$ in direction $v$ is less than $0$ when 

    \begin{equation}
        z^0 \notin \mc{B}(z^*, K_2 L^3 \epsilon^{1/4} \norm{z^*}{2}) \cup \mc{B}(-\kappa z^*, K_2 L^{13} \epsilon^{1/4} \norm{z^*}{2}) \cup \{0\}
    \end{equation}

    This implies that $z^0$ is not a stationary point and further that there exists a descent direction at that point. Thus, there must exist a $\mu > 0$ such that the local error bound condition holds at $z^0$. For example, when $L = 1$ and when the activation function has first derivative bounded away from $0$, then this value of $\mu$ will simply be the minimum singular value of $W_1^T R_{z^0}$ where $R_{z_0} = \diag(\sigma'(W_1 z^0))$. The authors of \cite{huang2021provably} demonstrate that for a subgradient descent algorithm, the iterates stay out of the basin of attraction for the spurious stationary point, and a descent direction still exists. Thus, along the optimization trajectory, the local error bound condition holds following the same logic as above. The convergence rate from Theorem \ref{thm:gi} then applies, which gives us a linear convergence rate to the global optimum. 
\end{proof}

\section{Additional Experiments}

\subsection{Synthetic Data Experimental Details} \label{app:synth}

To set up a realizable problem where the error bound parameter $\mu$ can be computed easily, we use the following setup for a generation of data, a classification network on this data, and a way to compute adversarial attacks given this network. 

First, we generate data $x \in \R^m$ from a one-layer GAN $x = G(z)$ with $G(z) = \sigma(Wz)$ and $W \in R^{m \times d}$. We use a leaky RELU activation function as $\sigma$.   

Next, we consider a binary classifier on this data of the form $sign(\psi(x))$ where 

\begin{equation}
	\psi(x) = \frac{1}{\sqrt{k}} \sum_{\ell = 1}^k a_\ell \sigma(w_\ell \cdot x).
\end{equation}

Here, the $w_\ell$ are i.i.d from $N\left(0, \frac{1}{m} I_m \right)$ and $a_\ell$ are uniform over $\{-1, 1\}$. We will consider single-step gradient-based attacks $\eta \grad \psi(x)$. With high probability, a single gradient step will flip the sign of the label, so we can easily find adversarial attacks \cite{bubeck2021single}. 

To create a realizable instance for the RED problem, we generate a training set $S_{tr} = \{G(z_i) : z_i \sim N(0, I_d) \}_{i =1}^{n_{\text{train}}}$ and similarly a testing set $S_{te}$. The attack dictionary $D_a$ contains single-step gradient attacks on $S_{tr}$. A realizable RED instance is then $x' = x_{te} + D_a c_a^*$ for $x_{te} \in S_{te}$ and some vector $c_a^*$. We run alternating gradient descent as in Section \ref{sec:red_unreg} to solve for $z^k$ and $c_a^k$. Since we have knowledge of the true $z^*$ and $c_a^*$ for a given problem instance, we can exactly compute the local error bound parameter $\mu$ for a given $z^k$ and $c_a^k$ on the optimization trajectory as $\mu^2 = \frac{\norm{\grad_z f(z^k, c_a^k)}{2}^2 + \norm{\grad_{c_a} f(z^k, c_a^k)}{2}^2}{\norm{\Delta z^k}{2}^2 + \norm{\Delta c_a^k}{2}^2}$. 

\subsection{Real Data Experimental Details}

\begin{table}[h]
\begin{center}
\begin{tabular}{cc}
\hline
Layer Type & Size \\
\hline
Convolution + ReLU & $3 \times 3 \times 32$ \\
Convolution + ReLU & $3 \times 3 \times 32$ \\
Max Pooling & $2 \times 2$ \\
Convolution + ReLU & $3 \times 3 \times 64$ \\
Convolution + ReLU & $3 \times 3 \times 64$ \\
Max Pooling & $2 \times 2$ \\
Fully Connected + ReLU & $200$ \\
Fully Connected + ReLU & $200$ \\
Fully Connected + ReLU & $10$ \\
\hline
\end{tabular}
\end{center}
\caption{Network Architecture for the MNIST and Fashion-MNIST dataset}
\label{table:archs}
\end{table}

\textbf{Attack Coefficient Algorithm.} In Algorithm \ref{alg:red}, we run 500 steps of alternating between updating $z$ and $c_a$. To update $c_a$, we also applied Nesterov acceleration. For the proximal step, we set the step size to be the inverse of the operator norm of $D_a^T D_a$ i.e. the Lipschitz constant of the gradient. To set the regularization parameter, we use the procedure from \cite{thaker2022reverse} i.e. compute the value of $\lambda$ such that the solution for $c_a$ is the all-zeros vector using the optimality conditions for the problem and then multiplying that value of $\lambda$ by a small constant (e.g. $0.35$ for our experiments). 

\textbf{MNIST.} Table \ref{table:archs} shows the network architecture for the MNIST dataset. This is trained using SGD for $50$ epochs with learning rate $0.1$, momentum $0.5$, and batch size $128$, identical to the architecture from \cite{thaker2022reverse}. 

All PGD adversaries were generated using the Advertorch library \cite{ding2019advertorch}. We use the same hyperparameters as the adversarial training baselines and as \cite{thaker2022reverse}. Specifically, the $\ell_\infty$ PGD adversary ($\epsilon = 0.3$) used a step size $\alpha = 0.01$ and was run for $100$ iterations. The $\ell_2$ PGD adversary ($\epsilon = 2$) used a step size $\alpha = 0.1$ and was run for $200$ iterations. The $\ell_1$ PGD adversary ($\epsilon = 10$) used a step size $\alpha = 0.8$ and was run for $100$ iterations. 

We use a pretrained DCGAN using the architecture from the standard Pytorch implementation \cite{pytorch}. The initialization of $z^0$ for Algorithm \ref{alg:red} is as follows: we first sample 10 random initializations for $z$ and for each, run 100 epochs of Defense-GAN training on $x'$ using the MSE loss. Then, we initialize $z^0$ to the vector that gives best MSE loss over the 10 random restarts. 

\textbf{Fashion-MNIST.} Table \ref{table:archs} shows the network architecture for the Fashion-MNIST dataset. The PGD adversaries have identical hyperparameters as on the MNIST dataset. We use a pretrained Wasserstein-GAN \cite{WGAN_repo} and use the same initialization scheme as for the MNIST dataset. 

\textbf{CIFAR-10.} The classification network used is the pretrained Wide Resnet from Pytorch. The $\ell_\infty$ PGD adversary ($\epsilon = 0.03)$ used a step size $\alpha = 0.003$ and was run for $100$ iterations. The $\ell_2$ PGD adversary ($\epsilon = 0.05$) used a step size $\alpha = 0.05$ and run for $200$ iterations. The $\ell_1$ PGD adversary ($\epsilon = 12$) used a step size $\alpha = 1$ and was run for $100$ iterations. We used a pretrained StyleGAN-XL \cite{sauer2022stylegan}. To initialize and invert in the space $\mc{W}+$, we sampled 10000 initializations of $z$ and one random class. In batch, we run $3000$ iterations of Defense-GAN for each initialization and learn a vector in $\mc{W}+$ as initialization for the latent space before running Algorithm \ref{alg:red}. 

\subsection{Non-Random Networks that satisfy Local Error Bound Condition}

We begin with the trivial observation that for a realizable GAN inversion problem, when there is no nonlinearity $\sigma$, any non-random network where $W$ has full (column) rank will satisfy the local error bound condition. This is because the gradient is equal to $\grad_z f(z^k) = W^T (x' - G(z^k)) = W^T (W\hat{z} - W z^k)$. Since $W$ is a tall matrix, we have that $W^T W$ is a full-rank matrix and thus the local error bound condition can be satisfied with $\mu$ as the minimum singular value of $W^T W$. However, when we have the nonlinearity, this is not necessarily the case. In this section, we provide several examples of non-random networks that still satisfy the local error bound condition. We provide examples in the 1-layer GAN inversion setting for simplicity i.e. $G(z) = \sigma(Wz)$ with $\sigma$ as the leaky RELU activation function, although these examples although work with the attack dictionary from Section \ref{app:synth} as well. For all examples below, the ground truth $x$ is generated as $x = G(\hat{z})$ where $\hat{z}$ is drawn from a standard normal. 

\begin{example} (2-D GAN Inversion) 
    We can slightly modify the example given in Section \ref{sec:exp:synth_data} in the following way. Consider a GAN with latent space dimension $d = 2$ and output dimension $m = 100$. Suppose that the rows of $W$ are spanned by $m$ vectors, which are either $\begin{bmatrix} -\sqrt{2}/2 & \sqrt{2}/2 \end{bmatrix} + \epsilon_i $ or $\begin{bmatrix} \sqrt{2}/2 & \sqrt{2}/2 \end{bmatrix} + \epsilon_i$ where $\epsilon_i \sim N(0, 0.2)$ for $i \in [1, \dots, m]$. This is roughly the same distribution as the example in Section \ref{sec:exp:synth_data} but with some slight perturbation of $W$ on the 2-d unit sphere. We observe that in this simple case as well, optimization always succeeds to the global minimizer, the landscape looks as benign as in Figure \ref{fig:leb_optim}, and the average value of $\mu$ computed in practice over 10 test examples is $\mu = 2.17$. Note that the value of $\mu$ is significantly higher than when $\epsilon_i = 0$. We conjecture that the randomness in $\epsilon_i$ leads to an improved landscape as $x' - G(z^k)$ is less likely to fall close to the nullspace of $W_1^T R_1$ (see Algorithm \ref{alg:red} for definition of $R_i$), which is the case when the local error bound condition is not met. 
\end{example} 

\begin{example} (GAN Inversion with Hadamard Matrices)
    We can also extend the previous example by looking at Hadamard matrices in general, which are square matrices with entries $1$ and $-1$ and whose rows are mutually orthogonal. Suppose we look at the Hadamard matrix of order $4$, which is:

    \begin{equation}
        W = \begin{bmatrix} 1 & 1 & 1 & 1 \\ 1 & -1 & 1 & -1 \\ 1 & 1 & -1 & -1 \\ 1 & -1 & -1 & 1 \end{bmatrix}
    \end{equation}

    Suppose that for a GAN inversion problem with $d = 4, m = 100$, the rows of $W$ are spanned by the 4 rows of this Hadamard matrix. In this case as well, we observe that over 50 runs, GAN inversion always converges to the global minimizer when $z^0$ is randomly initialized. Further, the average value of $\mu$ we observe is $1.07$. This value improves to $\mu = 2.61$ when we add $\epsilon_i$ to each row of $W$ for $\epsilon_i \sim N(0, 0.2)$ as in the previous example. The orthogonality property is likely an important property in ensuring a benign optimization landscape and having the local error bound property hold - note that by construction, the rows are not all mutually orthogonal, but they are spanned by a mutually orthogonal set.  
\end{example}

\begin{example} (GAN Inversion with Vandermonde Matrices) 
    Consider a GAN with latent space dimension $d = 2$ and $m = 100$. Let $W$ be a normalized Vandermonde matrix of dimension $m \times d$, i.e. each unnormalized row is $\begin{bmatrix} 1 & i \end{bmatrix}$ for $i \in 1, \dots, m$. For this matrix, over 50 runs, we always converge to the global minimizer with an average $\mu$ value of $0.64$. 
\end{example}

\subsection{Qualitative Results}

\begin{figure}[h]
    \centering
    \includegraphics[width=\textwidth]{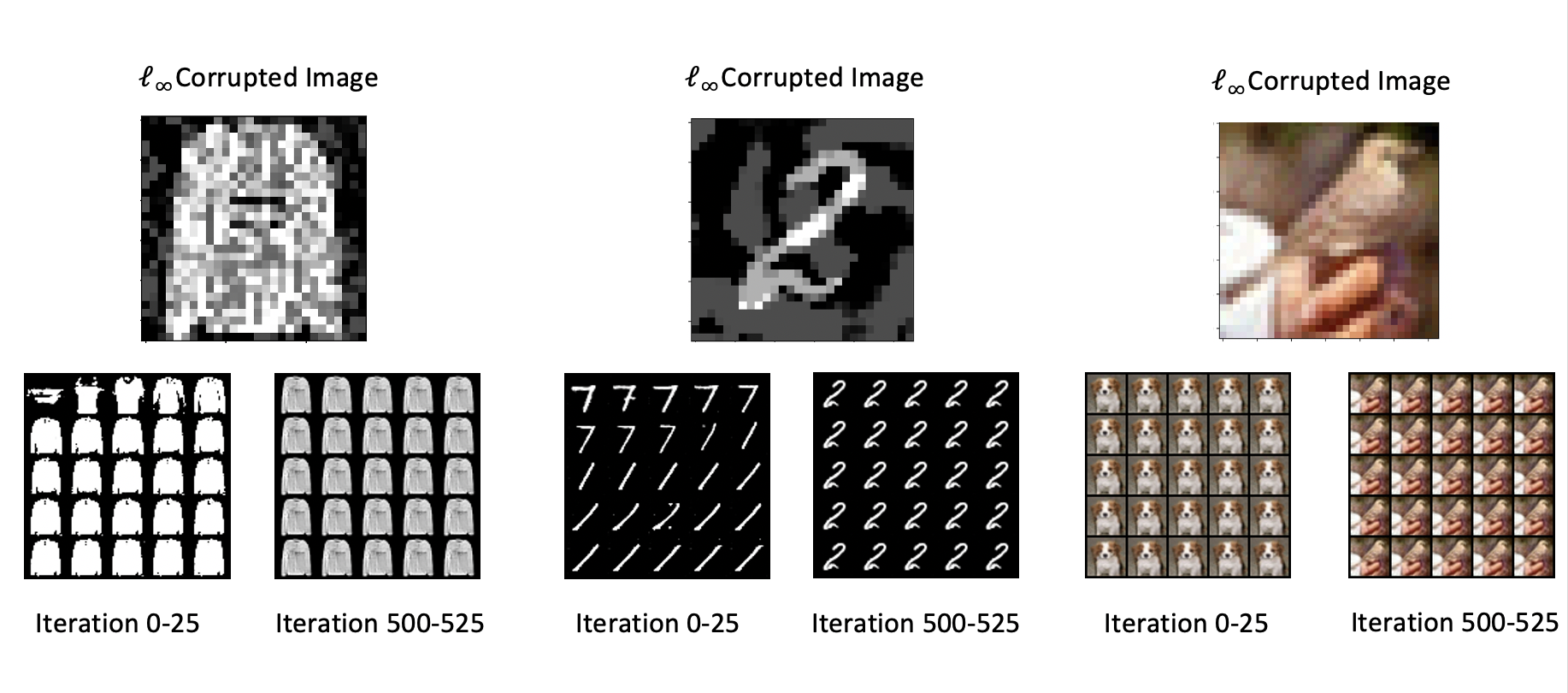}
    \caption{Qualitative Results for $\ell_\infty$ corrupted images for the MNIST, Fashion-MNIST, and CIFAR-10 datasets. Each $5 \times 5$ grid shows 25 iterations for $G(z^k)$ including the Defense-GAN initialization. }
    \label{fig:qual}
\end{figure}

To qualitatively get a sense of whether GAN inversion succeeds at recovering the true image, we plot $G(z^k)$ as a function of $k$ for the different datasets. We focus on $\ell_\infty$ attacks although we note that the denoised results look qualitatively identical for the different attacks (while the corrupted image looks different). If the GAN inversion succeeds at denoising and modelling the clean data, then we expect better attack detection accuracy since the $D_a c_a$ term can better capture the structure of the attack with limited noise. Figure \ref{fig:qual} shows the results on the 3 datasets: MNIST, Fashion-MNIST, and CIFAR-10. Each grid of images shows 25 images corresponding to $G(z^k)$ for 25 iterations. Note that the iteration number includes the number of iterations needed for Defense-GAN initialization. The Defense-GAN initialization looks qualitatively similar to a clean image, but the iterations after alternating between updating $z$ and $c_a$ allow us to further classify the attack. In all 3 examples, we see successful inversion of the image despite starting from an incorrect class, further supporting the benign optimization landscape of the inversion problem.

\end{document}